\documentclass[utf8]{FrontiersinHarvard} 

\usepackage{url,hyperref,lineno,microtype,subcaption,algorithmic,algorithm}
\usepackage[onehalfspacing]{setspace}

\newcommand{\G}{\mathcal{G}}
\newcommand{\hg}{\mathcal{H}}
\newcommand{\V}{\mathcal{V}}
\newcommand{\E}{\mathcal{E}}
\newcommand{\set}{\mathcal{S}}

\newcommand{\R}{\mathbb{R}}

\newcommand{\x}{\mathbf{x}}
\newcommand{\y}{\mathbf{y}}
\newcommand{\z}{\mathbf{z}}

\newcommand{\volume}{\mathrm{vol}}
\newcommand{\cut}{\mathrm{cut}}

\newcommand{\sign}{\mathrm{sgn}}

\newcommand{\lo}{\triangle}
\newcommand{\ncc}{\mathrm{NCC}}

\DeclareMathOperator*{\argmin}{argmin}

\newtheorem{definition}{Definition}
\newtheorem{theorem}{Theorem} 
\newtheorem{lemma}{Lemma}

\renewcommand{\algorithmiccomment}[1]{\bgroup\hfill~#1\egroup}

\def\keyFont{\fontsize{9}{11}\helveticabold }
\def\firstAuthorLast{} 
\def\Authors{Yu Zhu\,$^{*}$ and Santiago Segarra}
\def\Address{Department of Electrical and Computer Engineering, Rice University, Houston, TX, United States}
\def\corrAuthor{Yu Zhu}
\def\corrEmail{yz126@rice.edu}

\begin{document}
\onecolumn
\firstpage{1}

\title[Hypergraphs with Edge-Dependent Vertex Weights]{Hypergraphs with Edge-Dependent Vertex Weights: $p$-Laplacians and Spectral Clustering} 

\author[\firstAuthorLast ]{\Authors} 
\address{} 
\correspondance{} 

\extraAuth{}

\maketitle

\begin{abstract}

We study $p$-Laplacians and spectral clustering for a recently proposed hypergraph model that incorporates edge-dependent vertex weights (EDVW).
These weights can reflect different importance of vertices within a hyperedge, thus conferring the hypergraph model higher expressivity and flexibility.
By constructing submodular EDVW-based splitting functions, we convert hypergraphs with EDVW into submodular hypergraphs for which the spectral theory is better developed.
In this way, existing concepts and theorems such as $p$-Laplacians and Cheeger inequalities proposed under the submodular hypergraph setting can be directly extended to hypergraphs with EDVW.
For submodular hypergraphs with EDVW-based splitting functions, we propose an efficient algorithm to compute the eigenvector associated with the second smallest eigenvalue of the hypergraph $1$-Laplacian.
We then utilize this eigenvector to cluster the vertices, achieving higher clustering accuracy than traditional spectral clustering based on the $2$-Laplacian.
More broadly, the proposed algorithm works for all submodular hypergraphs that are graph reducible.
Numerical experiments using real-world data demonstrate the effectiveness of combining spectral clustering based on the $1$-Laplacian and EDVW.

\keyFont{ \section{Keywords:} submodular hypergraphs, $p$-Laplacian, spectral clustering, edge-dependent vertex weights, decomposable submodular function minimization} 
\end{abstract}

\section{Introduction}\label{s:intro}

Spectral clustering makes use of eigenvalues and eigenvectors of graph Laplacians to group vertices in a graph.
It is one of the most popular clustering methods due to its generality, efficiency, and strong theoretical basis.
Standard graph Laplacians were first adopted to obtain relaxations of balanced graph cut criteria~\citep{von2007tutorial}.
Later, these were generalized to $p$-Laplacians, which are able to provide better approximations of the Cheeger constant~\citep{amghibech2003eigenvalues, buhler2009spectral}.
Especially, the second smallest eigenvalue of the $1$-Laplacian is identical to the Cheeger constant and the partition that achieves the optimal Cheeger cut can be obtained by thresholding the corresponding eigenvector~\citep{szlam2010total}.
      
Graphs are widely used to model pairwise interactions, but in many real-world applications the entities engage in higher-order relationships~\citep{benson2016higher, schaub2021signal}.
For instance, in co-authorship networks multiple authors may interact in writing an article together~\citep{chitra2019random}.
In an e-commerce system, multiple customers can be associated if they once purchased the same product~\citep{li2018tail}.
In text mining, multiple documents are related to each other if they contain the same keywords~\citep{hayashi2020hypergraph, zhu2021co}.
Such multi-way relations can be modeled by hypergraphs, where the notion of an edge is generalized to a hyperedge that can connect more than two vertices.
       
In graphs, there is only one way to cut an edge, thus, a scalar weight is enough to characterize the cut.
But in hypergraphs, there may exist multiple ways to split a hyperedge.
Consequently, a splitting function $w_e$ is introduced for each hyperedge $e$ in the hypergraph, assigning a cost to every possible cut of $e$. 
For any $\set\subseteq e$, $w_e(\set)$ indicates the penalty of partitioning $e$ into $\set$ and $e\setminus\set$.
In particular, when $w_e$ is submodular for every hyperedge $e$, the corresponding model is termed as a submodular hypergraph which has desirable mathematical properties, making it convenient for theoretical analysis~\citep{li2017inhomogeneous, li2018submodular}.
A series of results in graph spectral theory including $p$-Laplacians, nodal domain theorems, and Cheeger inequalities have been generalized to submodular hypergraphs~\citep{li2018submodular, yoshida2019cheeger}. 

The choice of hyperedge splitting functions has a large practical effect on the hypergraph clustering performance.
There are mainly two types of splitting functions in existing work.
One is the so-called all-or-nothing splitting function in which an identical penalty is charged if the hyperedge is split regardless of how its vertices are separated~\citep{hein2013total}.
Another slightly more general type is the class of cardinality-based splitting functions where the splitting penalty depends only on the number of vertices placed on each side of the split~\citep{veldt2020hypergraph}.

The limitation of existing splitting functions is that they treat all the vertices in a hyperedge equally while in practice these vertices may have different degrees of contribution to the hyperedge.
Such information can be captured by edge-dependent vertex weights (EDVW): 
Every vertex $v$ is associated with a weight $\gamma_e(v)$ for each incident hyperedge $e$ that reflects the contribution or importance of $v$ to $e$~\citep{chitra2019random}.
Going back to the aforementioned examples, EDVW can be used to model the author positions in a co-authored article, the quantity of a product bought by a customer, as well as the frequency of a word in a document.

Spectral theory on the hypergraph model with EDVW is much less developed than on submodular hypergraphs.
Existing works studying hypergraphs with EDVW have only focused on random walk-based Laplacian matrices~\citep{chitra2019random, hayashi2020hypergraph, zhu2021co}, thus raising the question: \emph{How to define non-linear $p$-Laplacians for the hypergraph model that incorporates EDVW?}
Our basic idea for solving this problem is to convert a hypergraph with EDVW into a submodular hypergraph then $p$-Laplacians and related theorems developed for submodular hypergraphs can be directly leveraged.
Based on our earlier work~\citep{zhu2022hypergraphs}, the model conversion can be achieved by defining submodular EDVW-based splitting functions in the form of $w_e(\set)=g_e(\sum_{v\in\set}\gamma_e(v))$ where $g_e$ is a concave function and the splitting penalty $w_e(\set)$ is dependent only on the sum of EDVW in $\set$.
Moreover, hypergraphs with such splitting functions are proved to be graph reducible, meaning that there exists some graph sharing the same cut properties~\citep{zhu2022ans}.

Since the $1$-Laplacian provides the tightest approximation of the Cheeger constant, a follow-up question is: \emph{How to apply $1$-spectral clustering to submodular hypergraphs with EDVW-based splitting functions?} 
To this end, we develop an algorithm to compute the second eigenvector of the $1$-Laplacian for EDVW-based submodular hypergraphs based on the inverse power method (IPM).
The IPM was initially proposed for undirected graphs~\citep{hein2010inverse}, then generalized to submodular hypergraphs with cardinality-based splitting functions~\citep{li2018submodular}.  
A key to the success of IPM is an efficient solution to the inner-loop optimization problem in it.
In this paper, we derive an equivalent definition of graph reducibility based on which we further propose an efficient solution to the inner problem that works for all graph reducible submodular hypergraphs including those with EDVW.
The proposed solution can also be used to solve submodular function minimization (SFM) problems~\citep{bach2013learning} when the submodular function can be represented as sums of concave functions applied to modular functions.

The major contributions of this paper can be summarized as follows:
\begin{enumerate}
\item[(1)] We present an equivalent definition of graph reducibility in terms of the Lov\'asz extension of the cut function (see Theorem~\ref{THM:REDUCIBLE}), which is helpful for understanding the relations between graph Laplacians and hypergraph Laplacians.
\item[(2)] We propose an algorithm to compute the eigenvector of the $1$-Laplacian associated with the second smallest eigenvalue for all graph reducible submodular hypergraphs including those with EDVW-based splitting functions, and use the eigenvector for $1$-spectral clustering.
\item[(3)] We validate the effectiveness of the proposed algorithm which leverages both of EDVW and $1$-spectral clustering via numerical experiments on real-world datasets.
\end{enumerate}

\noindent\textbf{Paper outline.}
The rest of this paper is structured as follows.
Preliminary mathematical concepts and submodular hypergraphs are reviewed in Section~\ref{s:pre}.
Section~\ref{s:sumb_hg_edvws} introduces the hypergraph model with EDVW and shows how to convert it to graph reducible submodular hypergraphs by constructing EDVW-based splitting functions. 
The section also presents two equivalent definitions for graph reducibility.
The proposed $1$-spectral clustering algorithm is described in Section~\ref{s:1_spec_cluster}.
Section~\ref{s:exp} presents experimental results.
The relation between hypergraph Laplacians defined in different ways and the application of the proposed algorithm in SFM are discussed in Section~\ref{s:discuss}.
Closing remarks are included in Section~\ref{s:conclude}.

\noindent\textbf{Notation.}
For a vector $\x$ and a set $\set$, $\x_{\set}$ denotes the vector formed by the entries of $\x$ indexed by $\set$ and $\x(\set)=\sum_{v\in\set} x_v$.
The operator $\mathcal{P}_{a,b}(\x)$ projects every entry of $\x$ onto the range $[a,b]$.
Throughout the paper we assume that the considered hypergraphs are connected.

\section{Preliminaries}\label{s:pre}

\subsection{Mathematical preliminaries}\label{ss:math_pre}

For a finite set $\V$, a set function $F:2^\V\to\R$ is called submodular if $F(\set_1)+F(\set_2) \geq F(\set_1\cup\set_2)+F(\set_1\cap\set_2)$ for every $\set_1,\set_2\subseteq\V$.
Considering a set function $F:2^\V\to\R$ such that $F(\emptyset)=0$ where $\V=[N]=\{1,2,\cdots,N\}$, its Lov\'asz extension $f:\R^N\to\R$ is defined as follows.
For any $\x\in\R^N$, sort its entries in non-increasing order $x_{i_1} \geq x_{i_2} \geq \cdots \geq x_{i_N}$, where $(i_1,i_2,\cdots,i_N)$ is a permutation of $(1,2,\cdots,N)$, and set
	\begin{equation}\label{e:lovasz}
		f(\x) = \sum\nolimits_{j=1}^{N-1} F(\set_j)(x_{i_j}-x_{i_{j+1}}) + F(\V)x_{i_N},
	\end{equation}
where $\set_j=\{i_1,\cdots,i_j\}$ for $1\leq j<N$.
A set function $F$ is submodular if and only if its Lov\'asz extension $f$ is convex~\citep{lovasz1983submodular}.
For any $\set\subseteq\V$, $F(\set)=f(\mathbf{1}_\set)$ where $\mathbf{1}_\set$ is the indicator vector of $\set$.
If $F(\V)=0$, $f(a\x+b\mathbf{1})=af(\x)$ for any $a\in\R_{\geq 0}, b\in\R$.
More properties of submodular functions can be found in Appendix~\ref{s:proof_pre}.

\subsection{Submodular hypergraphs}\label{ss:subm_hg}

Let $\hg=(\V,\E,\mu,\{w_e\})$ denote a submodular hypergraph~\citep{li2018submodular} where $\V=[N]$ is the vertex set and $\E$ is the set of hyperedges. 
The function $\mu:\V\to\R_{+}$ assigns positive weights to vertices.
Each hyperedge $e\in\E$ is associated with a submodular splitting function $w_e:2^e\to\R_{\geq 0}$ that assigns non-negative penalties to every possible split of $e$.
Moreover, $w_e$ is required to satisfy $w_e(\emptyset)=0$ and be symmetric so that $w_e(\set)=w_e(e\setminus\set)$ for any $\set\subseteq e$.
The domain of $w_e$ can be extended from $2^e$ to $2^{\V}$ by setting $w_e(\set)=w_e(\set\cap e)$ for any $\set\subseteq\V$, guaranteeing that the submodularity is maintained.

A cut is a partition of the vertex set $\V$ into two disjoint, non-empty subsets denoted by $\set$ and its complement $\V\setminus\set$.
The weight of the cut is defined as the sum of splitting penalties associated with each hyperedge~\citep{li2018submodular, veldt2020hypergraph}, i.e., 
	\begin{equation}\label{e:hg_cut}
		\cut_{\hg}(\set) = \sum\nolimits_{e\in\E} w_e(\set).
	\end{equation}
The normalized Cheeger cut (NCC) is defined as
	\begin{equation}\label{e:hg_ncc}
		\ncc(\set) = \frac{\cut_{\hg}(\set)}{\min\{\volume(\set),\volume(\V\setminus\set)\}}
	\end{equation}
where $\volume(\set)=\sum_{v\in\set}\mu(v)$ denotes the volume of $\set$.
The $2$-way Cheeger constant is defined as the minimum NCC over all non-empty subsets of $\V$ except itself, i.e., 
	\begin{equation}\label{e:2_cheeger_constant}
		h_2 = \min_{\emptyset\subset\set\subset\V} \ncc(\set).
	\end{equation}
The solution to~\eqref{e:2_cheeger_constant} provides an optimal partitioning in the sense that we obtain two clusters which are balanced in terms of volume and loosely connected as captured by a small cut weight.
In this paper, we adopt the minimization of NCC as our objective.
There exist other clustering measures such as ratio cut, normalized cut and ratio Cheeger cut, which are closely related~\citep{von2007tutorial, buhler2009spectral}.

Optimally solving~\eqref{e:2_cheeger_constant} has been shown to be NP-hard for graphs~\citep{wagner1993between}, let alone hypergraphs.
In spectral graph theory, Cheeger inequalities are derived to bound and approximate the Cheeger constant using graph $p$-Laplacians~\citep{amghibech2003eigenvalues, buhler2009spectral, szlam2010total, chang2016spectrum, chang2017nodal, tudisco2018nodal}.
The results have been generalized to submodular hypergraphs~\citep{li2018submodular, yoshida2019cheeger}.
The $p$-Laplacian $\lo_p$ of a submodular hypergraph is defined as an operator that, for all $\x\in\R^N$, induces
	\begin{equation}\label{e:hg_Qp}
		\langle\x,\lo_p(\x)\rangle = \sum\nolimits_{e\in\E} \vartheta_e f_e(\x)^p \triangleq Q_p(\x),
	\end{equation}
where $\vartheta_e=\max_{\set\subseteq e}w_e(\set)$ and $f_e$ is the Lov\'asz extension of the normalized splitting function $\vartheta_e^{-1}w_e$.
Notice that $\lo_p$ can be alternatively defined in terms of the subdifferential of $f_e$~\citep{li2018submodular}, while we keep the definition~\eqref{e:hg_Qp} since it is more instrumental to our development.  
In particular, when $p=1$, $Q_1(\x)$ turns out to be the Lov\'asz extension of the cut function defined in~\eqref{e:hg_cut}.
It has been proved that the second smallest eigenvalue of the $1$-Laplacian is identical to the Cheeger constant $h_2$~\citep{li2018submodular}.

\section{EDVW-based Submodular Hypergraphs}\label{s:sumb_hg_edvws}

\subsection{The hypergraph model with EDVW}\label{ss:hg_edvws}

Let $\hg=(\V,\E,\mu,\kappa,\{\gamma_e\})$ represent a hypergraph with EDVW~\citep{chitra2019random} where $\V$, $\E$, and $\mu$ respectively denote the vertex set, the hyperedge set, and positive vertex weights.
The function $\kappa:\E\to\R_{+}$ assigns positive weights to hyperedges, and those weights can reflect the strength of connection. 
Each hyperedge $e\in\E$ is associated with a function $\gamma_e:\V\to\R_{\geq 0}$ to assign edge-dependent vertex weights.
For $v\in e$, $\gamma_e(v)$ is positive and measures the importance of the vertex $v$ to the hyperedge $e$; for $v\notin e$, $\gamma_e(v)$ is set to zero.
For convenience, we define $\gamma_e(\set)=\sum_{v\in\set}\gamma_e(v)$.

The motivation of introducing EDVW is to enable the hypergraph model to describe the cases when the vertices in the same hyperedge contribute differently to this hyperedge.
This information cannot be captured by hypergraphs adopting the all-or-nothing or cardinality-based splitting functions and is also hard to be directly described by submodular hypergraphs, but it can be conveniently represented via EDVW.
For example, \cite{chitra2019random} studies the application of ranking authors in an academic citation network where authors and papers are respectively modeled as vertices and hyperedges.
For a paper $e$ and any author $v$ of the paper, the corresponding EDVW $\gamma_e(v)$ is set to $2$ if $v$ is the first or last author, otherwise the weight is set to $1$.

\subsection{Building submodular hypergraphs from EDVW}\label{ss:w_edvws}

In order to effectively handle EDVW while still leveraging existing results obtained for submodular hypergraphs, we consider the conversion from a hypergraph with EDVW $\hg=(\V,\E,\mu,\kappa,\{\gamma_e\})$ to a submodular hypergraph $\hg=(\V,\E,\mu,\{w_e\})$.
The basic idea is to keep $\V$, $\E$, and $\mu$ unchanged, and construct submodular splitting functions $\{w_e\}$ from EDVW $\{\gamma_e\}$ and hyperedge weights $\kappa$.
After such a transformation, we can directly extend concepts such as $p$-Laplacians and related theorems proposed for the submodular hypergraph model~\citep{li2018submodular} to the EDVW-based hypergraph model. 

In our preliminary works~\citep{zhu2022hypergraphs, zhu2022ans}, we have proposed a class of submodular EDVW-based splitting functions in the following form
	\begin{equation}\label{e:w_edvws}
		w_e(\set) = h_e(\kappa(e)) \cdot g_e(\gamma_e(\set)),
	\end{equation} 
where $h_e: \R_{+}\to\R_{+}$ is an arbitrary function and $g_e: [0, \gamma_e(e)]\to\R_{\geq 0}$ is concave, symmetric with respect to $\gamma_e(e)/2$, and satisfies $g_e(0)=0$.
The resulting $w_e$ is a valid splitting function that is non-negative, submodular, symmetric, and satisfies $w_e(\emptyset)=0$.
In practice, it is reasonable to select a non-decreasing function for $h_e$ such as $h_e(x)=1$ and $h_e(x)=x$ since a larger hyperedge weight $\kappa(e)$ is expected to lead to a larger splitting penalty for the same split of the hyperedge.
Possible choices of $g_e$ include $g_e(x)=x\cdot(\gamma_e(e)-x)$, $g_e(x)=\min\{x,\gamma_e(e)-x\}$, and $g_e(x)=\min\{x,\gamma_e(e)-x,b\}$ where $b$ is a positive parameter. 
Also notice that for trivial EDVW, namely $\gamma_e(v)=1$ for all $v\in e$, the splitting functions defined as~\eqref{e:w_edvws} reduce to cardinality-based ones~\citep{veldt2020hypergraph}.

\subsection{Hypergraph-to-graph reductions}\label{ss:hg2g}

Submodular hypergraphs with splitting functions defined as~\eqref{e:w_edvws} have a desirable property that they are graph reducible.
In other words, they can project onto some graph which shares identical cut properties.
Following~\citep{veldt2020hypergraph}, we consider the reduction of a hypergraph to a (possibly directed) graph with a potentially augmented vertex set.
For a directed graph (digraph) $\G$ with vertex set $\V_\G$ and weighted adjacency matrix $\mathbf{A}$ whose entry $A_{uv}$ denotes the weight of the directed edge pointing from $u$ to $v$, its cut function is defined as 
	\begin{equation}\label{e:g_cut}
		\cut_\G(\set) = \sum\nolimits_{u\in\set, v\in\V_\G\setminus\set} A_{uv}.
	\end{equation}
In the following, we state a formal definition for graph reducibility in terms of cut weights, which is a variant of the definition in terms of hyperedge splitting functions stated in~\citep{veldt2020hypergraph}.
\begin{definition}
For a submodular hypergraph $\hg$ with vertex set $\V$, we say that its cut function $\cut_{\hg}(\set)$ is reducible to the cut function of a directed graph $\G$ with vertex set $\V_\G=\V\cup\bar{\V}$ where $\bar{\V}$ is a set of auxiliary vertices if the following equality holds 
	\begin{equation}\label{e:reduce1}
		\cut_{\hg}(\set) = \min_{\mathcal{T}\subseteq\bar{\V}} \cut_\G(\set\cup\mathcal{T}), \quad \forall \set\subseteq\V.
	\end{equation}
\end{definition}

In the following theorem, we show an equivalent condition for graph reducibility regarding the Lov\'asz extension of the cut function, which is beneficial for our later development of the $1$-spectral clustering algorithm. 
The Lov\'asz extension of the graph cut function can be written as
	\begin{equation}\label{e:g_cut_lovasz}
		Q_1^{(g)}(\y) = \sum\nolimits_{u,v\in\V_\G} A_{uv} \max\{y_u-y_v, 0\}
	\end{equation}
where $\y$ is a vector of length $\V_\G$.
\begin{theorem}\label{THM:REDUCIBLE}
The equality presented in~\eqref{e:reduce1} is equivalent to
	\begin{equation}\label{e:reduce2}
		Q_1(\x) = \min_{\bar{\x}\in\R^M} Q_1^{(g)}(\y), \quad \forall \,  \x\in\R^N
	\end{equation}
where $Q_1(\x)$ and $Q_1^{(g)}(\y)$ are respectively defined in \eqref{e:hg_Qp} and \eqref{e:g_cut_lovasz}, $|\V|=N$, $|\bar{\V}|=M$, and $\y\in\R^{N+M}$ is composed of $\y_\V=\x$ and $\y_{\bar{\V}}=\bar{\x}$.
\end{theorem}
\begin{proof}
The equivalence between~\eqref{e:reduce1} and~\eqref{e:reduce2} is proved in Appendix~\ref{s:proof_reducible}.
\end{proof}

It has been shown in~\citep{veldt2020hypergraph} that all hypergraphs with submodular cardinality-based splitting functions are graph reducible. 
Our earlier work~\citep{zhu2022ans} has generalized the conclusion to hypergraphs with submodular EDVW-based splitting functions.
In the following section, we propose a $1$-spectral clustering algorithm for all submodular hypergraphs that are graph reducible including those EDVW-based ones, which are the focus of this paper.

\begin{algorithm}[t]
\caption{$1$-spectral clustering for hypergraphs with EDVW}\label{alg:spec_cluster}
\begin{algorithmic}[1]
\vspace{0.2em}
\STATE \textbf{Input:} hypergraph with EDVW $\hg=(\V,\E,\mu,\kappa,\{\gamma_e\})$
\vspace{0.2em}
\STATE Convert $\hg$ to a submodular hypergraph by constructing submodular splitting functions based on~\eqref{e:w_edvws}
\vspace{0.2em}
\STATE Compute the second eigenvector of the hypergraph $1$-Laplacian via the minimization of $R_1(\x)$ in~\eqref{e:hg_R1}
\vspace{0.2em}
\STATE Threshold the obtained eigenvector to get the bipartition of $\V$ where we choose the threshold value as the one that minimizes the NCC in~\eqref{e:hg_ncc}
\vspace{0.2em}
\end{algorithmic}
\end{algorithm}

\section{Spectral Clustering based on the $1$-Laplacian}\label{s:1_spec_cluster}

\subsection{IPM-based $1$-spectral clustering}

We study spectral clustering algorithms for EDVW-based submodular hypergraphs leveraging the $1$-Laplacian.
As mentioned in Section~\ref{ss:subm_hg}, for submodular hypergraphs the Cheeger constant $h_2$ is equal to the second smallest eigenvalue $\lambda_2$ of the $1$-Laplacian $\lo_1$.
The corresponding optimal bipartition can be obtained by thresholding the eigenvector of $\lo_1$ associated with $\lambda_2$~\citep{li2018submodular}.
This eigenvector can be computed by minimizing 
	\begin{equation}\label{e:hg_R1}
		R_1(\x) = \frac{Q_1(\x)}{\min_{c\in\R}\|\x-c\mathbf{1}\|_{1,\mu}},
	\end{equation}
where $\|\x\|_{1,\mu}=\sum_{v\in\V}\mu(v)|x_v|$.
Given the eigenvector $\x$, a partitioning can be defined as $\set=\{v\in\V \,|\, x_v>t\}$ and its complement, where $t$ is a threshold value. 
The optimal $t$ can be determined as the one that minimizes the NCC in~\eqref{e:hg_ncc}.
The pipeline for the $1$-spectral clustering algorithm is summarized in Algorithm~\ref{alg:spec_cluster}.

\begin{algorithm}[t]
\caption{IPM-based minimization of $R_1(\x)$~\citep{hein2010inverse, li2018submodular}}\label{alg:ipm}
\begin{algorithmic}[1]
\vspace{0.2em}
\STATE \textbf{Input:} submodular hypergraph $\hg=(\V,\E,\mu,\{w_e\})$ with $N$ vertices, accuracy $\epsilon$
\vspace{0.2em}
\STATE \textbf{Initialization:} non-constant $\x\in\R^N$ subject to $0\in\argmin_c\|\x-c\mathbf{1}\|_{1,\mu}$, $\lambda\leftarrow R_1(\x)$
\vspace{0.2em}
\STATE \textbf{repeat} 
\vspace{0.2em}
\STATE $\forall v\in\V$, 
$
g_v \leftarrow 
\begin{cases}
\sign(x_v) \cdot \mu(v), & \text{if } x_v \neq 0 \\
\frac{\mu_{-}(\x)-\mu_{+}(\x)}{\mu_{0}(\x)} \cdot \mu(v), & \text{if } x_v=0
\end{cases}
$
\vspace{0.2em}
\STATE $\x \leftarrow \argmin_{\|\x\|\leq 1} Q_1(\x) - \lambda\langle\x,\pmb{g}\rangle$  \COMMENT{\bf(inner problem)}
\vspace{0.2em}
\STATE $c \leftarrow \argmin_c\|\x-c\mathbf{1}\|_{1,\mu}$
\vspace{0.2em}
\STATE $\x \leftarrow \x- c\mathbf{1}$
\vspace{0.2em}
\STATE $\lambda' \leftarrow \lambda, \lambda \leftarrow R_1(\x)$
\vspace{0.2em}
\STATE \textbf{until} $\frac{|\lambda-\lambda'|}{\lambda'}<\epsilon$
\vspace{0.2em}
\STATE \textbf{Output:} $\x$
\vspace{0.2em}
\end{algorithmic}
\end{algorithm}

The minimization of $R_1(\x)$ can be solved based on the inverse power method~\citep{hein2010inverse, li2018submodular}, as outlined in Algorithm~\ref{alg:ipm}.
Three functions are introduced: $\mu_{+}(\x)=\sum_{v\in\V:x_v>0}\mu(v)$, $\mu_{-}(\x)=\sum_{v\in\V:x_v<0}\mu(v)$ and $\mu_{0}(\x)=\sum_{v\in\V:x_v=0}\mu(v)$.
Although this algorithm cannot guarantee convergence to the second eigenvector, the objective $R_1(\x)$ is guaranteed to decrease and converge in the iterative process.
Moreover, if we start from some point $\x = \mathbf{1}_\set$ in Algorithm~\ref{alg:ipm} where $(\set,\V\setminus\set)$ is a given partition of $\V$, then each step of the IPM-based method gives a partition that has a smaller (or at least equal) NCC value (see Theorem 4.2 in~\citep{hein2011beyond}).
This implies that we can leverage the partition obtained via other methods as initialization.
The algorithm was first proposed for the undirected graph setting~\citep{hein2010inverse}, then generalized to submodular hypergraphs with cardinality-based splitting functions~\citep{li2018submodular}.
It is actually a special case of a more general class of minimization algorithms called RatioDCA and generalized RatioDCA proposed in~\citep{hein2011beyond} and~\citep{tudisco2018community} in order to handle more types of balanced graph cuts and modularity measures, respectively.
The major difference between the graph setting and the hypergraph setting lies in how the inner-loop optimization problem (cf. line 5 in Algorithm~\ref{alg:ipm}) is solved.

In~\citep{li2018submodular}, the authors solved the inner problem using a random coordinate descent method~\citep{ene2015random} together with a divide-and-conquer algorithm proposed in~\citep{jegelka2013reflection}. 
The computational complexity of the divide-and-conquer algorithm depends on the time of solving the problem $\min_{\set\subseteq e} F(\set) \triangleq w_e(\set) + \z(\set)$ for an arbitrary vector $\z\in\R^{|e|}$.
For a cardinality-based splitting function, the solution to this problem can be found efficiently via a line search even when $|e|$ is large. 
In the line search, we create a series of sets $\set_0=\emptyset,\set_1,\cdots,\set_{|e|}$, where $\set_i$ contains $i$ vertices corresponding to the first $i$ smallest entries in vector $\z$.
We compare their objective values $F(\set_i)$ and identify the solution $\set_{i^\ast}$ leading to the minimum objective value. 
However, such solution does not work for EDVW-based splitting functions.
In the following section, we study the inner problem considering the EDVW-based case.

\subsection{Solution to the inner problem}

We propose efficient solutions to the inner problem for EDVW-based submodular hypergraphs leveraging the property that they are graph reducible.
More generally speaking, the proposed solutions work for all graph reducible submodular hypergraphs.
We first show that the inner problem is equivalent to another optimization problem~\eqref{e:inner_equiv_1} defined on the digraph obtained via the reduction.

\begin{theorem}\label{THM:INNER_EQUIV}
For any submodular hypergraph $\hg$ with vertex set $\V$, if it is reducible to a digraph $\G$ with vertex set $\V_{\G}=\V\cup\bar{\V}$ and edge set $\E_{\G}$, i.e., \eqref{e:reduce1} or \eqref{e:reduce2} holds, then the solution $\x$ to the inner problem in Algorithm~\ref{alg:ipm} can be obtained (up to a scalar multiple) by setting $\x=\y_{\V}$ and $\y\in\R^{|\V_\G|}$ is the solution to 
	\begin{equation}\label{e:inner_equiv_1}
		\min_{\|\y\|_2 \leq 1} Q_1^{(g)}(\y) - \langle\y, \tilde{\pmb{g}}\rangle	
	\end{equation}
where $\tilde{\pmb{g}}\in\R^{|\V_\G|}$ is a vector composed of $\tilde{\pmb{g}}_{\V}=\lambda\pmb{g}$ and $\tilde{\pmb{g}}_{\bar{\V}}=\mathbf{0}$.
\end{theorem}
\begin{proof}
	The proof can be found in Appendix~\ref{s:proof_inner_equiv} where we have used the equivalent definition for graph reducibility proposed in Theorem~\ref{THM:REDUCIBLE}. 
\end{proof}

We present two ways for solving problem~\eqref{e:inner_equiv_1} in Sections~\ref{sss:fista} and~\ref{sss:pdhg} respectively.

\subsubsection{Solving the inner problem via FISTA}\label{sss:fista}

Both of the original inner problem and its equivalent problem~\eqref{e:inner_equiv_1} are convex but non-smooth.
Inspired by~\citep{hein2010inverse}, we derive a dual formulation~\eqref{e:inner_equiv_2} of problem~\eqref{e:inner_equiv_1}.
Compared to the primal problem, the objective function $\Psi$ of the dual problem is smooth.
Moreover, problem~\eqref{e:inner_equiv_2} can be efficiently solved using a fast iterative shrinkage-thresholding algorithm (FISTA)~\citep{beck2009fast,nesterov1983method}, which has a guaranteed convergence rate $O(1/k^2)$ where $k$ is the number of iterations.
FISTA requires an upper bound on the Lipschitz constant of the gradient of $\Psi$ as the input, which is provided in~\eqref{e:Lipschitz}.
The steps of FISTA are summarized in Algorithm~\ref{alg:fista}.

To make it clear, in the theorem below, the set $\tilde{\E}_{\G}$ (as well as the directed edge set $\E_{\G}$) contains \emph{ordered} node pairs, meaning that $(u,v)$ and $(v,u)$ are different.
The parameter $m$ can be understood as the number of connected node pairs where the connection might be unidirectional or bidirectional.

\begin{theorem}\label{THM:INNER_DUAL}
\normalfont
Define a set 
$\tilde{\E}_{\G} = \{(u,v)\in\V_{\G}\times\V_{\G} \,|\, (u,v)\in\E_{\G} \text{ or } (v,u)\in\E_{\G}\}$
and set $m = |\tilde{\E}_{\G}| / 2$.
The dual of problem~\eqref{e:inner_equiv_1} is 
	\begin{equation}\label{e:inner_equiv_2}
		\min_{\substack{\pmb{\alpha}\in [0,1]^m \\ \alpha_{uv}+\alpha_{vu}=1}} \Psi(\pmb{\alpha}) \triangleq \| f_A(\pmb{\alpha}) - \tilde{\pmb{g}} \|_2^2
	\end{equation}
where $\pmb{\alpha}$ is a vector of length $m$ collecting all $\alpha_{uv}$ satisfying $(u,v)\in\tilde{\E}_{\G}$ and $u<v$.
For the function $f_A:\R^m\to\R^{N+M}$, the $u$th element of $f_A(\pmb{\alpha})$ is  
	\begin{equation*}
		[f_A(\pmb{\alpha})]_u = \sum\nolimits_{v \,|\, (u,v) \in \tilde{\E}_{\G}} A_{uv}\alpha_{uv} - A_{vu}\alpha_{vu}.
	\end{equation*}
The primal and dual variables are related as 
	\begin{equation}\label{e:linear_unit_ball}
		\y = - \frac{f_A(\pmb{\alpha}) - \tilde{\pmb{g}}}{\| f_A(\pmb{\alpha}) - \tilde{\pmb{g}} \|_2}.
	\end{equation}
The Lipschitz constant of the gradient of $\Psi$ is upper bounded by
	\begin{equation}\label{e:Lipschitz}
		L = 4 \max_{u\in\V_{\G}} \sum\nolimits_{v\in\V_{\G}} (A_{uv}+A_{vu})^2.
	\end{equation}
\end{theorem}
\begin{proof}
	The proof is given in Appendix~\ref{s:proof_inner_dual}.	
\end{proof}


In a nutshell, to solve the inner problem, we first compute the adjacency matrix of the digraph according to the graph reduction procedure proposed in~\citep{zhu2022ans}.
Then we solve the dual problem~\eqref{e:inner_equiv_2} using FISTA, and get the solution $\y$ to the primal problem~\eqref{e:inner_equiv_1} according to the relation between the primal and dual variables as shown in~\eqref{e:linear_unit_ball}.
Finally, we obtain the solution $\x$ to the original problem by taking the entries of $\y$ indexed by $\V$.

\begin{algorithm}[t]
\caption{Solution of problem \eqref{e:inner_equiv_2} with FISTA}\label{alg:fista}
\begin{algorithmic}[1]
\vspace{0.2em}
\STATE \textbf{Input:} A Lipschitz constant $L$ of $\nabla\Psi$
\vspace{0.2em}
\STATE \textbf{Initialization:} $\pmb{\alpha}=\pmb{\beta}\in\R^m$, $t=1$
\vspace{0.2em}
\STATE \textbf{repeat} 
\vspace{0.2em}
\STATE $\pmb{\alpha}' \leftarrow \pmb{\alpha}$, $\pmb{\alpha} \leftarrow \pmb{\beta} - \frac{1}{L}\nabla\Psi(\pmb{\beta})$
\vspace{0.2em}
\STATE $\pmb{\alpha} \leftarrow \mathcal{P}_{0,1}(\pmb{\alpha})$
\vspace{0.2em}
\STATE $t' \leftarrow t$, $t \leftarrow \frac{1 + \sqrt{1+4t^2}}{2}$
\vspace{0.2em}
\STATE $\pmb{\beta} \leftarrow \pmb{\alpha} + \frac{t'-1}{t}(\pmb{\alpha}-\pmb{\alpha}')$
\vspace{0.2em}
\STATE \textbf{until} convergence or a predefined maximum number of iterations is reached
\vspace{0.2em}
\STATE \textbf{Output:} $\pmb{\alpha}$
\vspace{0.2em}
\end{algorithmic}
\end{algorithm}

\subsubsection{Solving the inner problem via PDHG}\label{sss:pdhg}

Problem~\eqref{e:inner_equiv_1} can also be solved using a primal-dual hybrid gradient (PDHG) algorithm~\citep{chambolle2011first, chambolle2016introduction, chambolle2016ergodic}.
Although both FISTA and PDHG ensure a quadratic convergence rate, it has been observed that PDHG can outperform FISTA in practice for clustering applications~\citep{hein2011beyond}.

PDHG is able to solve problems in the following form:
\begin{equation}\label{e:pdhg}
	\min_{\y\in\R^{N_1}} f_1(\y) + f_2(\mathbf{B}\y),	
\end{equation}
whose dual problem is
\begin{equation}\label{e:pdhg_dual}
	\max_{\z\in\R^{N_2}} -f_1^\ast(-\mathbf{B}^\top\z) - f_2^\ast(\z),	
\end{equation}
where $f_1:\R^{N_1}\to(-\infty,+\infty]$ and $f_2:\R^{N_2}\to(-\infty,+\infty]$ are proper, convex, lower semicontinuous functions, $\mathbf{B}:\R^{N_1}\to\R^{N_2}$ is a bounded linear operator, and $f_1^\ast$ and $f_2^\ast$ are the corresponding conjugate functions of $f_1$ and $f_2$~\citep{boyd2004convex}.

The solution to problem~\eqref{e:inner_equiv_1} can be obtained via normalizing the solution to problem~\eqref{e:inner_equiv_3} below.
\begin{equation}\label{e:inner_equiv_3}
	\min_{\y\in\R^{|\V_\G|}} Q_1^{(g)}(\y) + \frac{1}{2} \|\y-\tilde{\pmb{g}}\|_2^2.
\end{equation}
We can fit~\eqref{e:inner_equiv_3} into the form~\eqref{e:pdhg} by setting $N_1=|\V_\G|$, $N_2=|\E_\G|$, $f_1(\y)=\frac{1}{2}\|\y-\tilde{\pmb{g}}\|_2^2$, $f_2(\z)=\sum_{i=1}^{N_2}\max\{z_i,0\}$, and $\mathbf{B}$ is a $|\E_\G| \times |\V_\G|$ matrix.
For the row of $\mathbf{B}$ corresponding to edge $u \to v$, the $u$th and $v$th elements in the row are respectively equal to $A_{uv}$ and $-A_{uv}$, and the other elements are zero.
We can show that $f_1^\ast(\y) = \frac{1}{2}\|\y\|_2^2 + \langle\y, \tilde{\mathbf{g}}\rangle$, $f_2^\ast(\z) = 0$ with the domain that $0 \leq z_i \leq 1$ for any $0 \leq i \leq N_2$.
Since $f_1$ is $1$-strongly convex, we can leverage an accelerated variant of the PDHG algorithm. 
The algorithm tailored for problem~\eqref{e:inner_equiv_3} is given in Algorithm~\ref{alg:pdhg} (cf. Algorithm 8 in~\citep{chambolle2016introduction}).




\begin{algorithm}[t]
\caption{Solution of problem~\eqref{e:inner_equiv_3} with accelerated PDHG}\label{alg:pdhg}
\begin{algorithmic}[1]
\vspace{0.2em}
\STATE \textbf{Initialization:} $\tau\sigma\leq\frac{1}{\|\mathbf{B}\|_2^2}$, $\y=\bar{\y}\in\R^{|\V_\G|}$, $\z\in\R^{|\E_\G|}$
\vspace{0.2em}
\STATE \textbf{repeat} 
\vspace{0.2em}
\STATE $\z \leftarrow \mathcal{P}_{0,1}(\z+\sigma\mathbf{B}\bar{\y})$
\vspace{0.2em}
\STATE $\y' \leftarrow \y$, $\y \leftarrow \frac{1}{1+\tau}\left(\y-\tau(\mathbf{B}^\top\z-\tilde{\pmb{g}})\right)$
\vspace{0.2em}
\STATE $\theta \leftarrow \frac{1}{\sqrt{1+\tau}}$, $\tau \leftarrow \theta\tau$, $\sigma \leftarrow \frac{\sigma}{\theta}$
\vspace{0.2em}
\STATE $\bar{\y} \leftarrow \y + \theta(\y-\y')$
\vspace{0.2em}
\STATE \textbf{until} convergence or a predefined maximum number of iterations is reached
\vspace{0.2em}
\STATE \textbf{Output:} $\y$
\vspace{0.2em}
\end{algorithmic}
\end{algorithm}

\subsection{A special case: reduction to an undirected graph}

If the submodular hypergraph is reducible to an undirected graph (see Theorem 3.3 in~\citep{zhu2022ans} for examples), then there exists another dual formulation of problem~\eqref{e:inner_equiv_1} that bears a similar form to~\eqref{e:inner_equiv_2} while the gradient of its objective has a smaller Lipschitz constant (cf. Lemma 4.3 in~\citep{hein2010inverse}).
In fact, for this case, the hypergraph $1$-spectral clustering can be implemented via its graph counterpart.

Following Theorem~\ref{THM:REDUCIBLE}, we further define a vertex weight function $\mu_\G$ for the graph $\G$ such that $\mu_\G(v)=\mu(v)$ for $v\in\V$ and $\mu_\G(v)=0$ for each auxiliary vertex $v\in\bar{\V}$.
Then it follows from~\eqref{e:reduce2} that $R_1(\x) = \min_{\bar{\x}\in\R^M} R_1^{(g)}(\y)$ for any $\x\in\R^N$ where $R_1^{(g)}(\y) = \frac{Q_1^{(g)}(\y)}{\min_{c\in\R}\|\y-c\mathbf{1}\|_{1,\mu_\G}}$ and $\y$ is composed of $\y_\V=\x$ and $\y_{\bar{\V}}=\bar{\x}$.
Minimizing both sides of the equation over $\x\in\R^N$ leads to the same minimum NCC value.
When $\G$ is undirected, the second eigenvector of the graph $1$-Laplacian can be computed by minimizing $R_1^{(g)}(\y)$ and then the second eigenvector of the hypergraph $1$-Laplacian can be computed as $\x=\y_{\V}$.
This can also be understood in terms of the cut function.
Following~\eqref{e:reduce1}, it is easy to show that the Cheeger constant of the graph $\G$ is identical to that of the hypergraph $\hg$.
This coincides with Theorem 3.3 in~\citep{liu2021strongly}, which further proves that if $\set$ is the set in $\G$ leading to the minimum NCC for the vertex weight function $\mu_\G$, then $\set\cap\V$ is the minimum NCC set in $\hg$.


\section{Experiments}\label{s:exp}

We evaluate the performance of the proposed $1$-spectral clustering algorithm for EDVW-based submodular hypergraphs (termed as EDVW-based hereafter) by focusing on the bipartition case.
In particular, for hyperedge splitting functions as in~\eqref{e:w_edvws}, we select $h_e(x)=x$ and $g_e(x)=\min\{x,\gamma(e)-x,\beta\gamma(e)\}$ for every hyperedge $e\in\E$, namely
	\begin{equation}\label{e:w_used}
		w_e(\set) = \kappa(e) \cdot \min\{\gamma_e(\set), \gamma_e(e)-\gamma_e(\set), \beta\gamma_e(e)\}
	\end{equation}
where $\beta$ is tunable.
Notice that we reuse the symbols $\alpha$ and $\beta$ in this section, which are different from and not related to $\pmb{\alpha}$ and $\pmb{\beta}$ used when describing FISTA.
A hypergraph with splitting functions as~\eqref{e:w_used} can be reduced to a digraph that consists of $|\V|+2|\E|$ vertices and $|\E| + 2\sum_{e\in\E}|e|$ edges.
More precisely, we first project each hyperedge $e$ onto a small graph which contains $|e|+2$ vertices including two auxiliary vertices denoted by $e'$ and $e''$.
For every $v\in e$, there are two directed edges respectively from $v$ to $e'$ and from $e''$ to $v$, both of weight $\kappa(e)\gamma_e(v)$. 
There is also a directed edge of weight $\beta\kappa(e)\gamma_e(e)$ from $e'$ to $e''$.
Then we concatenate these small graphs for all hyperedges together to form the final graph.

\noindent\textbf{Datasets.}
We consider three widely used real-world datasets.

\emph{Reuters Corpus Volume 1 (RCV1):} 
This dataset is a collection of manually categorized newswire stories~\citep{lewis2004rcv1}.
We consider two categories C14 and C23.
A few short documents containing less than $20$ words are ignored.
We select the $100$ most frequent words in the corpus after removing stop words and words appearing in $>10\%$ and $<0.2\%$ of the documents.
We then remove documents containing less than $5$ selected words, leaving us with $7446$ documents.
A document (vertex) belongs to a word (hyperedge) if the word appears in the document.
The edge-dependent vertex weights are taken as the corresponding tf-idf (term frequency-inverse document frequency) values~\citep{leskovec2020mining} to the power of $\alpha$, where $\alpha$ is a tunable parameter.

\emph{$20$ Newsgroups~\footnote{\url{http://qwone.com/~jason/20Newsgroups/}}:} 
This is also a text dataset. 
For our $2$-partition case, we consider the documents in categories `rec.autos' and `rec.sport.hockey'.
We preprocess the dataset following the same procedure as used in RCV1 above.
We finally construct a hypergraph of $1389$ vertices and $100$ hyperedges.

\emph{Covtype:}
This dataset contains patches of forest that are in different cover types. 
We consider two cover types (labeled as $4$ and $5$) and all numerical features.
Each numerical feature is first quantized into $20$ bins of equal range and then mapped to hyperedges.
The resulting hypergraph has $12240$ vertices and $196$ hyperedges. 
For each hyperedge (bin), we compute the distance between each feature value in this bin and their median, and then normalize these distances to the range $[0,1]$.
The edge-dependent vertex weights are computed as $\exp(-\alpha\cdot \text{distance})$.
Under this setting, larger edge-dependent vertex weights are assigned to vertices whose feature values are close to the typical feature value in the corresponding bin.

Following~\citep{hayashi2020hypergraph}, for all datasets we set the hyperedge weight $\kappa(e)$ to the standard deviation of the EDVW $\gamma_e(v)$ for all $v\in\V$.
Following~\citep{li2018submodular}, we set the vertex weight $\mu(v)$ to the vertex degree defined in the submodular hypergraph model, i.e., $\mu(v)=\sum_{e\ni v}\vartheta_e$.

\begin{figure}[t!]
\centering
\includegraphics[scale=0.5]{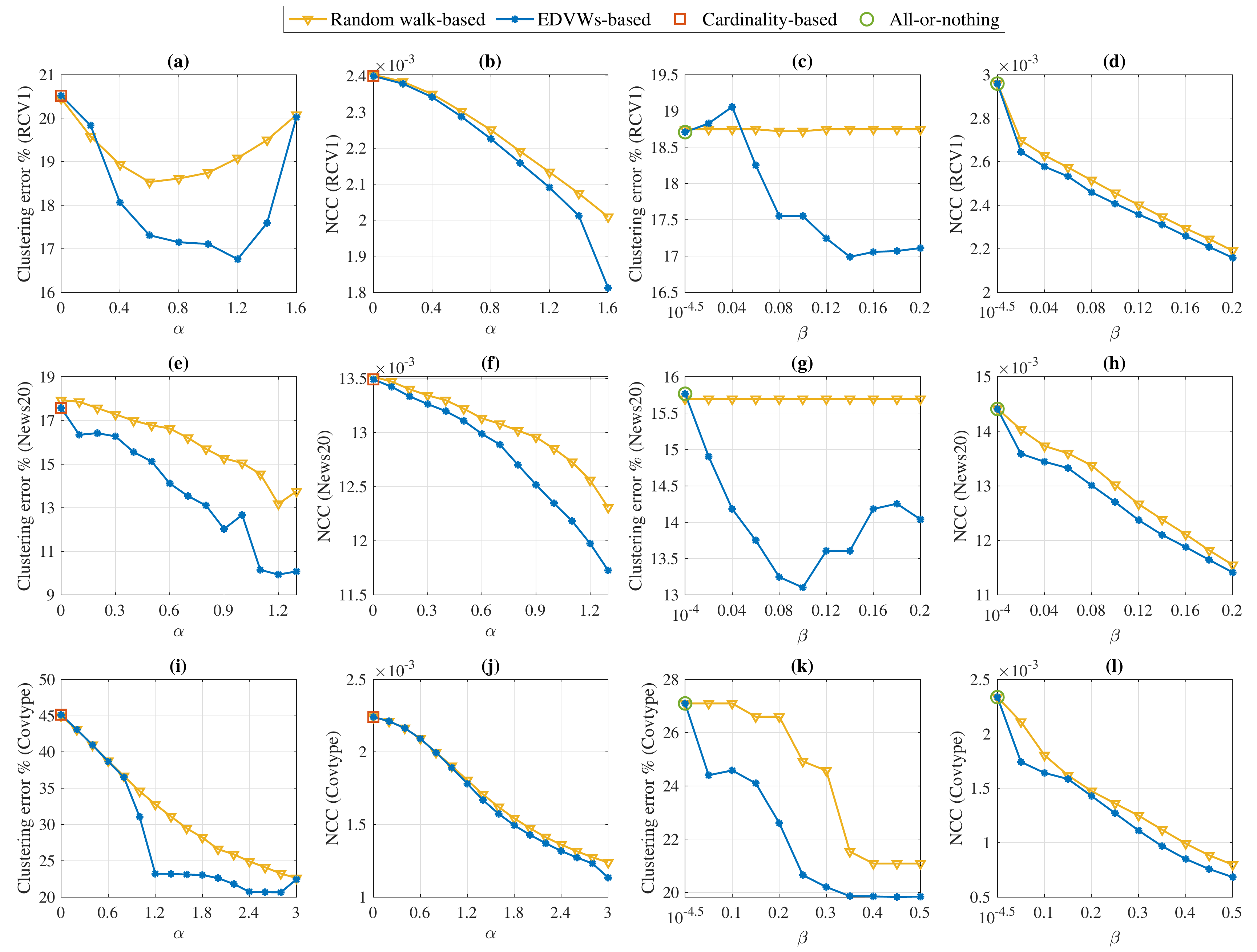}	
\caption{Clustering performance for three real-world datasets displayed in pairs of figures depicting the clustering error and the NCC value as a function of the parameters $\alpha$ or $\beta$. For RCV1, we fix $\beta=0.2$ in (a-b) and $\alpha=1$ in (c-d). For $20$ Newsgroups, we fix $\beta=0.1$ in (e-f) and $\alpha=0.8$ in (g-h). For Covtype, we fix $\beta=0.2$ in (i-j) and $\alpha=2$ in (k-l). A proper choice of $\alpha$ and $\beta$ helps significantly decrease the clustering error compared with existing methods. The performance improvement may benefit from both the use of EDVW and $1$-Laplacian.}
\label{fig:res}
\end{figure}

\noindent\textbf{Baselines.}
We compare the proposed approach with three baseline methods.

\emph{Random walk-based:}
The paper~\citep{hayashi2020hypergraph} defines a hypergraph Laplacian matrix based on random walks with EDVW. 
We compute the second eigenvector of the normalized hypergraph Laplacian (cf. (6) in~\citep{hayashi2020hypergraph}) and then threshold it to get the partitioning.

\emph{Cardinality-based:}
In the description of the datasets above, when $\alpha=0$, we get the trivial EDVW and the splitting functions reduce to cardinality-based ones. 

\emph{All-or-nothing:}
For the adopted splitting functions~\eqref{e:w_used}, they reduce to the all-or-nothing case if $\beta$ is small enough (i.e., $\beta \leq \min_{v\in e}\gamma_e(v)$).

For the proposed method, in Algorithm~\ref{alg:ipm} we adopt the eigenvector obtained in the random walk-based method described above as the starting point.
We solve the inner problem using PDHG presented in Algorithm~\ref{alg:pdhg}.
Moreover, in Algorithm~\ref{alg:pdhg} we initialize $\tau = \sigma = \frac{0.9}{\|\mathbf{B}\|_2}$ and $\mathbf{z} = \mathbf{0}$.

\noindent\textbf{Results.}
The results are shown in Figure~\ref{fig:res}.
We present both the clustering error and the NCC where the clustering error is computed as the fraction of incorrectly clustered samples.
For the RCV1 dataset (a-d), we fix $\beta=0.2$ to observe the influence of $\alpha$ in (a-b) and fix $\alpha=1$ to test $\beta$ in (c-d). 
For the 20 Newsgroups dataset (e-h), we fix $\beta=0.1$ in (e-f) and $\alpha=0.8$ in (g-h) to observe the effects of $\alpha$ and $\beta$, respectively.
For the Covtype dataset (i-l), we fix $\beta=0.2$ in (i-j) and $\alpha=2$ in (k-l).
We can see that for all the considered datasets, when edge-dependent vertex weights are ignored (including $\alpha=0$ for cardinality-based splitting functions and $\beta$ close to zero for the all-or-nothing splitting function), the clustering performance is severely deteriorated, validating that it is necessary to incorporate EDVW for extra modeling flexibility.
It can also be observed that, when appropriate values for $\alpha$ and $\beta$ are selected (such as $\alpha=1.2, \beta=0.2$ for RCV1, $\alpha = 1.2, \beta=0.1$ for 20 Newsgroups and $\alpha=2.4, \beta=0.2$ for Covtype), the proposed method performs much better than the random walk-based method which depends on the classical Laplacian. 
This highlights the value of using the non-linear $1$-Laplacian in spectral clustering.
To summarize, both the use of EDVW and $1$-Laplacian are beneficial for improving the spectral clustering performance.

\section{Related Work}\label{s:discuss}

We discuss the related works in this section. 
We will show the relationship between hypergraph Laplacians introduced via random walks and those defined based on submodular splitting functions.
We will also show how the proposed solution to the inner problem contributes to SFM.

\noindent\textbf{Random walk-based Laplacians.}
Here we expand the discussion about the relation between hypergraph Laplacians based on random walks with EDVW and those studied in this paper considering submodular EDVW-based splitting functions.

Considering the hypergraph model with EDVW (cf. Section~\ref{ss:hg_edvws}), the random walk incorporating EDVW is defined as follows~\citep{chitra2019random}.
Starting at some vertex $u$, the walker selects a hyperedge $e$ containing $u$ with probability proportional to $\kappa(e)$, then moves to vertex $v$ contained in $e$ with probability proportional to $\gamma_e(v)$.
In this process, the probability of moving from $u$ to $v$ via $e$ is
\begin{equation}\label{e:p_uev}
P_{u\to e\to v} = 
\begin{cases}
\frac{\kappa(e)}{\sum_{e'\ni u}\kappa(e')}\cdot\frac{\gamma_e(v)}{\gamma_e(e)}, & u,v\in e, \\
0, & \text{else}.
\end{cases}
\end{equation}
Then one can define a transition matrix $\mathbf{P}$ whose $(u,v)$th entry denotes the transition probability from $u$ to $v$ and is computed as $P_{uv}=\sum_{e\in\E}P_{u\to e\to v}$.
When the hypergraph is connected, the random walk converges to a unique stationary distribution $\pmb{\pi}$ which is the all-positive dominant left eigenvector of $\mathbf{P}$ scaled to satisfy $\|\pmb{\pi}\|_1=1$.
In~\citep{chitra2019random, hayashi2020hypergraph}, hypergraph Laplacians based on such random walks are proposed and they are actually equal to combinatorial and normalized graph Laplacian matrices (i.e., graph $2$-Laplacians) of an undirected graph which is defined on the same vertex set as the hypergraph and has the following adjacency matrix
	\begin{equation}\label{e:rw_adj}
	\mathbf{A} = \frac{\mathbf{\Phi P} + \mathbf{P \Phi}}{2}	
	\end{equation}
where $\mathbf{\Phi}$ is a diagonal matrix whose $(v,v)$th entry is $\pi_v$.

Consider a hypergraph $\hg$ with submodular splitting functions in the following form for each of its hyperedges:
	\begin{equation}\label{e:rw_w}
	w_e(\set) = \sum\nolimits_{u\in\set, v\in e\setminus\set} \frac{\pi_{u}P_{u\to e\to v}+\pi_{v}P_{v\to e\to u}}{2}, \quad \forall \set\subseteq e.
	\end{equation}
It is easy to show that $\hg$ is reducible to the graph $\G$ defined by the adjacency matrix~\eqref{e:rw_adj} since they have the same cut function as $\cut_{\hg}(\set) = \cut_\G(\set) = \sum\nolimits_{u\in\set, v\in\V\setminus\set} \frac{\pi_{u}P_{uv}+\pi_{v}P_{vu}}{2}$.
Notice that there are no auxiliary vertices introduced in the graph reduction.
Following Theorem~\ref{THM:REDUCIBLE}, we have $Q_1(\x) = Q_1^{(g)}(\x)$, implying that the hypergraph $1$-Laplacian of $\hg$ is identical to the graph $1$-Laplacian of $\G$ if we assume that they share the same vertex weight function $\mu$.

\noindent\textbf{Decomposable submodular function minimization.}
In the inner problem of Algorithm~\ref{alg:ipm}, if the norm of $\x$ is $\|\x\|_{\infty}$, the problem is equivalent to the following SFM problem 
	\begin{equation}\label{e:dsfm}
		\min_{\set\subseteq\V} \sum\nolimits_{e\in\E} w_e(\set) - \lambda \mathbf{g}(\set),
	\end{equation}
where the primal and dual variables are related as $x_v=1$ if $v\in\set$ and $x_v=-1$ if $v\notin\set$~\citep{li2018submodular}.
Hence, the proposed solution to the inner problem can also be used to solve problems in the form of~\eqref{e:dsfm} when each function $w_e$ can be represented as a concave function applied to a modular (additive) function (cf.~\eqref{e:w_edvws}).

\section{Conclusion}\label{s:conclude}

We presented an equivalent definition of graph reducibility based on which we further proposed a $1$-spectral clustering algorithm for submodular hypergraphs that are graph reducible, especially for those with EDVW-based splitting functions.
Through experiments on real-world datasets, we showcased the value of combining the hypergraph 1-Laplacian and EDVW.
Future research directions include:
(1) Developing computation methods for the hypergraph 1-Laplacian’s eigenvectors which can work efficiently for all submodular splitting functions,
(2) Designing multi-way partitioning algorithms based on non-linear Laplacians~\citep{buhler2009spectral}, and 
(3) Exploring applications of $p$-Laplacians for different values of $p$~\citep{fu2022p}.

\appendix
\section*{Appendix}
\renewcommand{\thesection}{\Alph{section}}
\section{Preliminary Proofs}\label{s:proof_pre}

In this section, we present some properties of submodular functions which are useful in the derivations of the proposed theorems.

\begin{lemma}[Proposition 3.7 in~\citep{bach2013learning}]
Let $F:2^\V\to\R$ be a submodular function such that $\V=[N]$ and $F(\emptyset)=0$, and denote its Lov\'asz extension by $f$, then
	\begin{equation}\label{e:F_f_equiv}
		\min_{\set\subseteq\V} F(\set) = \min_{\x\in \{0,1\}^N} f(\x) = \min_{\x\in [0,1]^N} f(\x).
	\end{equation}
Moreover, the set of minimizers of $f(\x)$ on $[0,1]^N$ is the convex hull of minimizers of $f(\x)$ on $\{0,1\}^N$.
\end{lemma}

In the following, we use $G$ to denote a submodular function defined on the power set of $\V\cup\bar{\V}$ where $|\V|=N$, $|\bar{\V}|=M$ and $\V\cap\bar{\V}=\emptyset$.
We assume that $G(\emptyset)=0$ and denote its Lov\'asz extension by $g(\y)=g(\x,\bar{\x})$ for any $\y=[\x^{\top},\bar{\x}^{\top}]^{\top}$ where $\x\in\R^N$ and $\bar{\x}\in\R^M$ respectively correspond to $\V$ and $\bar{\V}$.

\begin{lemma}[Proposition B.3 in~\citep{bach2013learning}]\label{lemma:partial_min}
Define a set function $F:2^{\V}\to\R$ based on $G$ as 
	\begin{equation}\label{e:F_G_2}
		F(\set) = \min_{\mathcal{T}\subseteq\bar{\V}} G(\set\cup\mathcal{T}) - \min_{\mathcal{T}\subseteq\bar{\V}} G(\mathcal{T})
	\end{equation}
for any $\set\subseteq\V$.
The function $F$ is submodular and satisfies $F(\emptyset)=0$.
If $\min_{\mathcal{T}\subseteq\bar{\V}} G(\mathcal{T}) = 0$, the Lov\'asz extension of $F$ is such that for all $\x\in\R_{\geq 0}^N$, 
	\begin{equation}\label{e:f_g_2}
		f(\x) = \min_{\bar{\x}\in\R_{\geq 0}^M} g(\x,\bar{\x}).
	\end{equation}
\end{lemma}

\begin{lemma}\label{lemma:partial_equiv}
Further assume that the submodular function $G$ is non-negative and satisfies $G(\V\cup\bar{\V})=0$.
Given some $\x\in\R^N$ whose minimum element and maximum element are respectively $a$ and $b$, the inequality 
	\begin{equation}
		g(\x,\bar{\x}) \geq g(\x,\mathcal{P}_{a,b}(\bar{\x}))
	\end{equation}
holds for any $\bar{\x}\in\R^M$.
This implies that 
	\begin{equation}
		\min_{\bar{\x}\in\R^M} g(\x,\bar{\x}) = \min_{\bar{\x}\in [a', b']^M} g(\x,\bar{\x}) = \min_{\bar{\x}\in [a, b]^M} g(\x,\bar{\x})
	\end{equation}
for any $a'<a$ and $b'>b$.
\end{lemma}
\begin{proof}
The proof follows immediately from the definition of the Lov\'asz extension and thus is omitted here.
\end{proof}

\begin{lemma}\label{lemma:partial_min_extend}
Following Lemma~\ref{lemma:partial_min}, if $G$ further satisfies the conditions listed in Lemma~\ref{lemma:partial_equiv}, then the Lov\'asz extension of $F$ is such that for all $\x\in\R^N$,
	\begin{equation}
		f(\x) = \min_{\bar{\x}\in\R^M} g(\x,\bar{\x}).
	\end{equation}
\end{lemma}
\begin{proof}
Given some $\x\in\R^N$ whose minimum element is $a$, we have
\begin{equation*}
	f(\x) 
	= f(\x-a\mathbf{1}) 
	\overset{(a)}{=} \min_{\bar{\x}\in\R_{\geq 0}^M} g(\x-a\mathbf{1},\bar{\x}) 
	= \min_{\bar{\x}\in\R_{\geq 0}^M} g(\x,\bar{\x}+a\mathbf{1}) 
	= \min_{\bar{\x}\in [a,\infty)^M} g(\x,\bar{\x}) 
	\overset{(b)}{=} \min_{\bar{\x}\in\R^M} g(\x,\bar{\x})
\end{equation*}
where $(a)$ followed~\eqref{e:f_g_2} and (b) followed Lemma~\ref{lemma:partial_equiv}.
\end{proof}

\begin{lemma}\label{lemma:partial}
Define a set function $F:2^{\bar{\V}}\to\R$ based on $G$ as 
	\begin{equation}\label{e:F_G_1}
		F(\mathcal{T}) = G(\set\cup\mathcal{T}) - G(\set)
	\end{equation}
for any $\mathcal{T}\subseteq\bar{\V}$.
The function $F$ is submodular and satisfies $F(\emptyset)=0$.
The Lov\'asz extension of $F$ is such that for all $\bar{\x}\in [0,1]^M$, 
	\begin{equation}\label{e:f_g_1}
		f(\bar{\x}) = g(\mathbf{1}_{\set},\bar{\x}) - G(\set).
	\end{equation}
\end{lemma}
\begin{proof}
For any $\mathcal{T}_1,\mathcal{T}_2\subseteq\bar{\V}$, one has
\begin{align*}
	& F(\mathcal{T}_1\cup\mathcal{T}_2) + F(\mathcal{T}_1\cap\mathcal{T}_2) \\
	\overset{(a)}{=} & G(\set\cup[\mathcal{T}_1\cup\mathcal{T}_2]) + G(\set\cup[\mathcal{T}_1\cap\mathcal{T}_2]) - 2G(\set) \\
	\overset{(b)}{=} & G([\set\cup\mathcal{T}_1]\cup[\set\cup\mathcal{T}_2]) + G([\set\cup\mathcal{T}_1]\cap[\set\cup\mathcal{T}_2]) - 2G(\set) \\
	\overset{(c)}{\leq} & G(\set\cup\mathcal{T}_1) + G(\set\cup\mathcal{T}_2) - 2G(\set) \\
	\overset{(d)}{=} & F(\mathcal{T}_1) + F(\mathcal{T}_2)
\end{align*}
where $(a)$ and $(d)$ followed~\eqref{e:F_G_1}, $(b)$ leveraged properties of set operations, and $(c)$ was from the submodularity of the function $G$. 
Hence, the function $F$ is proved to be submodular.

For any $\bar{\x}\in [0,1]^M$, sort its entries in non-increasing order $1 \geq \bar{x}_{i_1} \geq \bar{x}_{i_2} \geq \cdots \geq \bar{x}_{i_M} \geq 0$ and define $\mathcal{T}_j = \{i_1,\cdots,i_j\}$ for $1 \leq j < M$, then we can write
\begin{align*}
	f(\bar{\x}) 
	&\overset{(e)}{=} \sum\nolimits_{j=1}^{M-1} F(\mathcal{T}_j)(\bar{x}_{i_j}-\bar{x}_{i_{j+1}}) + F(\bar{\V})\bar{x}_{i_M}  \\
	&\overset{(f)}{=} \sum\nolimits_{j=1}^{M-1} \left[G(\set\cup\mathcal{T}_j)-G(\set)\right](\bar{x}_{i_j}-\bar{x}_{i_{j+1}}) + \left[G(\set\cup\bar{\V})-G(\set)\right]\bar{x}_{i_M}  \\
	&= -G(\set)\bar{x}_{i_1} + \sum\nolimits_{j=1}^{M-1} G(\set\cup\mathcal{T}_j)(\bar{x}_{i_j}-\bar{x}_{i_{j+1}}) + G(\set\cup\bar{\V})\bar{x}_{i_M} \\
	&= G(\set)(1-\bar{x}_{i_1}) + \sum\nolimits_{j=1}^{M-1} G(\set\cup\mathcal{T}_j)(\bar{x}_{i_j}-\bar{x}_{i_{j+1}}) + G(\set\cup\bar{\V})(\bar{x}_{i_M}-0) - G(\set) \\
	&\overset{(g)}{=} g(\mathbf{1}_{\set},\bar{\x}) - G(\set)
\end{align*} 
where $(e)$ and $(g)$ were obtained according to the definition of the Lov\'asz extension, and $(f)$ followed~\eqref{e:F_G_1}, thus~\eqref{e:f_g_1} is proved. 
\end{proof}


\section{Proof of Theorem~\ref{THM:REDUCIBLE}}\label{s:proof_reducible}

We will prove the equivalence between~\eqref{e:reduce1} and~\eqref{e:reduce2}.

Prove \eqref{e:reduce1} $\to$ \eqref{e:reduce2}:
The proof follows Lemmas~\ref{lemma:partial_min} and~\ref{lemma:partial_min_extend}.

Prove \eqref{e:reduce2} $\to$ \eqref{e:reduce1}:
We define a set function $F:2^{\bar{\V}}\to\R$ as 
	\begin{equation}\label{e:F_auxi}
		F(\mathcal{T}) = \cut_\G(\set\cup\mathcal{T}) - \cut_\G(\set)
	\end{equation}
for any $\mathcal{T}\subseteq\bar{\V}$.
According to Lemma~\ref{lemma:partial}, the function $F$ is submodular and its Lov\'asz extension $f$ is such that for all $\bar{\x}\in [0,1]^M$,
	\begin{equation}\label{e:f_auxi}
		f(\bar{\x}) = Q_1^{(g)}(\mathbf{1}_{\set},\bar{\x}) - \cut_\G(\set).
	\end{equation}
We can write
\begin{align*}
	\min_{\mathcal{T}\subseteq\bar{\V}} \cut_\G(\set\cup\mathcal{T}) 
	&\overset{(a)}{=} \min_{\mathcal{T}\subseteq\bar{\V}} F(\mathcal{T}) + \cut_\G(\set) \\
	&\overset{(b)}{=} \min_{\bar{\x}\in [0,1]^M} f(\bar{\x}) + \cut_\G(\set) \\
	&\overset{(c)}{=} \min_{\bar{\x}\in [0,1]^M} Q_1^{(g)}(\mathbf{1}_{\set},\bar{\x}) \\
	&\overset{(d)}{=} \min_{\bar{\x}\in \R^M} Q_1^{(g)}(\mathbf{1}_{\set},\bar{\x}) \\
	&\overset{(e)}{=} Q_1(\mathbf{1}_{\set}) = \cut_{\hg}(\set)
\end{align*}
where $(a)$ followed~\eqref{e:F_auxi}, $(b)$ followed~\eqref{e:F_f_equiv}, $(c)$ followed~\eqref{e:f_auxi}, $(d)$ followed Lemma~\ref{lemma:partial_equiv}, and $(e)$ followed~\eqref{e:reduce2}, thus~\eqref{e:reduce1} is obtained.

\section{Proof of Theorem~\ref{THM:INNER_EQUIV}}\label{s:proof_inner_equiv}

Notice that the choice of the norm in the inner problem of Algorithm~\ref{alg:ipm} only influences the scale of the solution.
If we select the infinity norm, then we have
\begin{align*}
	& \min_{\|\x\|_{\infty}\leq 1} Q_1(\x) - \lambda\langle\x,\pmb{g}\rangle \\
	\overset{(a)}{=} & \min_{\|\x\|_{\infty}\leq 1} \min_{\bar{\x}\in\R^M} Q_1^{(g)}(\x, \bar{\x}) - \lambda\langle\x,\pmb{g}\rangle \\
	= & \min_{\bar{\x}\in\R^M}\min_{\|\x\|_{\infty}\leq 1} Q_1^{(g)}(\x, \bar{\x}) - \lambda\langle\x,\pmb{g}\rangle \\
	\overset{(b)}{=} & \min_{\|\y\|_{\infty}\leq 1} Q_1^{(g)}(\y) - \langle\y,\tilde{\pmb{g}}\rangle 
\end{align*}
where $(a)$ followed \eqref{e:reduce2} and $(b)$ followed Lemma~\ref{lemma:partial_equiv} where we rewrite $Q_1^{(g)}(\y)=Q_1^{(g)}(\x, \bar{\x})$ for $\y=[\x^{\top},\bar{\x}^{\top}]^{\top}$ and $\tilde{\pmb{g}}=[\lambda\pmb{g}^{\top}, \mathbf{0}_{1\times M}]^{\top}$.
Then we replace the infinity norm with the Euclidean norm which only influences the scale of the solution.


\section{Proof of Theorem~\ref{THM:INNER_DUAL}}\label{s:proof_inner_dual}

According to the Lov\'asz extension of the cut function of a digraph, $Q_1^{(g)}(\y)$ can be rewritten as
\begin{align*}
Q_1^{(g)}(\y) 
&= \sum\nolimits_{u,v\in\V_{\G}} A_{uv} \max\{y_u-y_v, 0\} \\
&= \max_{\substack{\pmb{\alpha}\in [0,1]^m \\ \alpha_{uv}+\alpha_{vu}=1}} \sum\nolimits_{(u,v) \in \tilde{\E}_{\G}} A_{uv} (y_u-y_v) \alpha_{uv} \\
&= \max_{\substack{\pmb{\alpha}\in [0,1]^m \\ \alpha_{uv}+\alpha_{vu}=1}} \sum\nolimits_{(u,v) \in \tilde{\E}_{\G}} (A_{uv}\alpha_{uv}-A_{vu}\alpha_{vu}) y_u \\
&= \max_{\substack{\pmb{\alpha}\in [0,1]^m \\ \alpha_{uv}+\alpha_{vu}=1}} \sum\nolimits_{u\in\V_{\G}} y_u \left( \sum\nolimits_{v \,|\, (u,v) \in \tilde{\E}_{\G}} A_{uv}\alpha_{uv}-A_{vu}\alpha_{vu} \right) \\
&= \max_{\substack{\pmb{\alpha}\in [0,1]^m \\ \alpha_{uv}+\alpha_{vu}=1}} \langle \y, f_A(\pmb{\alpha}) \rangle
\end{align*}
It follows that
\begin{align*}
  & \min_{\|\y\|_2\leq 1} Q_1^{(g)}(\y) - \langle\y,\tilde{\pmb{g}}\rangle \\
= & \min_{\|\y\|_2\leq 1} \max_{\substack{\pmb{\alpha}\in [0,1]^m \\ \alpha_{uv}+\alpha_{vu}=1}} \langle \y, f_A(\pmb{\alpha}) \rangle - \langle\y,\tilde{\pmb{g}}\rangle \\
\overset{(a)}{=} & \max_{\substack{\pmb{\alpha}\in [0,1]^m \\ \alpha_{uv}+\alpha_{vu}=1}} \min_{\|\y\|_2\leq 1} \langle \y, f_A(\pmb{\alpha}) \rangle - \langle\y,\tilde{\pmb{g}}\rangle \\
= & \max_{\substack{\pmb{\alpha}\in [0,1]^m \\ \alpha_{uv}+\alpha_{vu}=1}} \min_{\|\y\|_2\leq 1} \langle \y, f_A(\pmb{\alpha}) - \tilde{\pmb{g}}\rangle \\
\overset{(b)}{=} & \max_{\substack{\pmb{\alpha}\in [0,1]^m \\ \alpha_{uv}+\alpha_{vu}=1}} -\| f_A(\pmb{\alpha}) - \tilde{\pmb{g}} \|_2 
\end{align*}
where $(a)$ followed Corollary 37.3.2 in~\citep{rockafellar1970convex}, and in $(b)$ we used the solution to the minimization of the linear function over the Euclidean unit ball given by \eqref{e:linear_unit_ball}.
Hence, the dual problem~\eqref{e:inner_equiv_2} is derived.

We rewrite the objective function $\Psi(\pmb{\alpha})$ of the dual problem as
\begin{small}
\begin{equation*}
\Psi(\pmb{\alpha}) = \sum\nolimits_{u\in\V_\G} \left( \left( \sum\nolimits_{v \,|\, (u,v) \in \tilde{\E}_\G} A_{uv}\alpha_{uv} - A_{vu}\alpha_{vu} \right) - \tilde{g}_u \right)^2. 
\end{equation*}
\end{small}%
It can be observed that the terms in $\Psi(\pmb{\alpha})$ involving a specific $\alpha_{ij}$ are
\begin{small}
\begin{align*}
& \left( \left( \sum\nolimits_{v \,|\, (i,v) \in \tilde{\E}_\G} A_{iv}\alpha_{iv} - A_{vi}\alpha_{vi} \right) - \tilde{g}_i \right)^2 + \left( \left( \sum\nolimits_{v \,|\, (j,v) \in \tilde{\E}_\G} A_{jv}\alpha_{jv} - A_{vj}\alpha_{vj} \right) - \tilde{g}_j \right)^2 \\
= & \left( \left( \sum\nolimits_{v \,|\, (i,v) \in \tilde{\E}_\G} (A_{iv}+A_{vi})\alpha_{iv}-A_{vi} \right) - \tilde{g}_i \right)^2 + \left( \left( \sum\nolimits_{v \,|\, (j,v) \in \tilde{\E}_\G} A_{jv}-(A_{jv}+A_{vj})\alpha_{vj} \right) - \tilde{g}_j \right)^2
\end{align*}
\end{small}%
where we replaced $\alpha_{vi}$ and $\alpha_{jv}$ with $1-\alpha_{iv}$ and $1-\alpha_{vj}$, respectively.
It follows that 
\begin{small}
\begin{equation*} 
\frac{\partial \Psi(\pmb{\alpha})}{\partial \alpha_{ij}} = 2(A_{ij}+A_{ji}) \left( \left( \sum\nolimits_{v \,|\, (i,v) \in \tilde{\E}_\G} (A_{iv}+A_{vi})\alpha_{iv}-A_{vi} \right) - \tilde{g}_i + \left( \sum\nolimits_{v \,|\, (j,v) \in \tilde{\E}_\G} (A_{jv}+A_{vj})\alpha_{vj}-A_{jv} \right) + \tilde{g}_j \right).   
\end{equation*}
\end{small}%
Thus, we have
\begin{small}
\begin{equation*} 
\frac{\partial \Psi(\pmb{\alpha})}{\partial \alpha_{ij}} - \frac{\partial \Psi(\pmb{\beta})}{\partial \beta_{ij}} 
= 2(A_{ij}+A_{ji}) \left( \sum\nolimits_{v \,|\, (i,v) \in \tilde{\E}_\G} (A_{iv}+A_{vi})(\alpha_{iv}-\beta_{iv}) + \sum\nolimits_{v \,|\, (j,v) \in \tilde{\E}_\G} (A_{jv}+A_{vj})(\alpha_{vj}-\beta_{vj}) \right).   
\end{equation*}
\end{small}%
It follows that 
\begin{small}
\begin{align*}  
& \|\nabla\Psi(\pmb{\alpha}) - \nabla\Psi(\pmb{\beta})\|_2^2 
= \sum_{\substack{(i,j)\in\tilde{\E}_\G \\ i<j}} \left(\frac{\partial \Psi(\pmb{\alpha})}{\partial \alpha_{ij}} - \frac{\partial \Psi(\pmb{\beta})}{\partial \beta_{ij}} \right)^2\\
&= 4\sum_{\substack{(i,j)\in\tilde{\E}_\G \\ i<j}}(A_{ij}+A_{ji})^2 \left( \sum_{v \,|\, (i,v) \in \tilde{\E}_\G} (A_{iv}+A_{vi})(\alpha_{iv}-\beta_{iv}) + \sum_{v \,|\, (j,v) \in \tilde{\E}_\G} (A_{jv}+A_{vj})(\alpha_{vj}-\beta_{vj}) \right)^2 \\
&\overset{(a)}{\leq} 8\sum_{\substack{(i,j)\in\tilde{\E}_\G \\ i<j}}(A_{ij}+A_{ji})^2 \left( \left(\sum_{v \,|\, (i,v) \in \tilde{\E}_\G} (A_{iv}+A_{vi})(\alpha_{iv}-\beta_{iv})\right)^2 + \left(\sum_{v \,|\, (j,v) \in \tilde{\E}_\G} (A_{jv}+A_{vj})(\alpha_{vj}-\beta_{vj})\right)^2 \right) \\
&\overset{(b)}{\leq} 8\sum_{\substack{(i,j)\in\tilde{\E}_\G \\ i<j}}(A_{ij}+A_{ji})^2 \left( \sum_{v \,|\, (i,v) \in \tilde{\E}_\G} (A_{iv}+A_{vi})^2 \!\!\!\! \sum_{v \,|\, (i,v) \in \tilde{\E}_\G} (\alpha_{iv}-\beta_{iv})^2  + \!\!\!\! \sum_{v \,|\, (j,v) \in \tilde{\E}_\G} (A_{jv}+A_{vj})^2 \!\!\!\! \sum_{v \,|\, (j,v) \in \tilde{\E}_\G} (\alpha_{jv}-\beta_{jv})^2 \right) \\
&= 8\sum_{(i,j)\in\tilde{\E}_\G}(A_{ij}+A_{ji})^2 \left( \sum_{v \,|\, (i,v) \in \tilde{\E}_\G} (A_{iv}+A_{vi})^2 \sum_{v \,|\, (i,v) \in \tilde{\E}_\G} (\alpha_{iv}-\beta_{iv})^2  \right) \\
&= 8\sum_{i\in\V_\G} \left( \sum_{v \,|\, (i,v) \in \tilde{\E}_\G} (A_{iv}+A_{vi})^2 \right)^2 \sum_{v \,|\, (i,v) \in \tilde{\E}_\G} (\alpha_{iv}-\beta_{iv})^2 \\
&\leq 8\left( \max_{i\in\V_\G} \sum_{v \,|\, (i,v) \in \tilde{\E}_\G} (A_{iv}+A_{vi})^2 \right)^2 \sum_{(i,v) \in \tilde{\E}_\G} (\alpha_{iv}-\beta_{iv})^2 \\
&= 16\left( \max_{i\in\V_\G} \sum_{v \,|\, (i,v) \in \tilde{\E}_\G} (A_{iv}+A_{vi})^2 \right)^2 \|\pmb{\alpha}-\pmb{\beta}\|_2^2    
\end{align*}
\end{small}%
where $(a)$ leveraged the inequality $(a+b)^2 \leq 2(a^2+b^2)$ and $(b)$ followed the Cauchy-Schwarz inequality.
An upper bound on the Lipschitz constant of $\nabla\Psi$ is thus obtained.


%

%

\section*{Funding}

This work was supported by NSF under award CCF 2008555.



\section*{Data Availability Statement}
The data and code for this study are available at~\url{https://github.com/yuzhu2019/hg_edvws_1spectralclustering}.


\bibliographystyle{Frontiers-Harvard} 
\bibliography{references.bib}

\begin{thebibliography}{39}
\providecommand{\natexlab}[1]{#1}
\expandafter\ifx\csname urlstyle\endcsname\relax
  \providecommand{\doi}[1]{doi:\discretionary{}{}{}#1}\else
  \providecommand{\doi}{doi:\discretionary{}{}{}\begingroup
  \urlstyle{rm}\Url}\fi
\providecommand{\selectlanguage}[1]{\relax}
\providecommand{\bibAnnoteFile}[1]{%
  \IfFileExists{#1}{\begin{quotation}\noindent\textsc{Key:} #1\\
  \textsc{Annotation:}\ \input{#1}\end{quotation}}{}}
\providecommand{\bibAnnote}[2]{%
  \begin{quotation}\noindent\textsc{Key:} #1\\
  \textsc{Annotation:}\ #2\end{quotation}}

\bibitem[{Amghibech(2003)}]{amghibech2003eigenvalues}
Amghibech, S. (2003).
\newblock Eigenvalues of the discrete p-laplacian for graphs.
\newblock \emph{Ars Combinatoria} 67, 283--302
\bibAnnoteFile{amghibech2003eigenvalues}

\bibitem[{Bach(2013)}]{bach2013learning}
Bach, F. (2013).
\newblock Learning with submodular functions: A convex optimization
  perspective.
\newblock \emph{Foundations and Trends{\textregistered} in Machine Learning} 6,
  145--373.
\newblock \url{https://doi.org/10.1561/2200000039}
\bibAnnoteFile{bach2013learning}

\bibitem[{Beck and Teboulle(2009)}]{beck2009fast}
Beck, A. and Teboulle, M. (2009).
\newblock A fast iterative shrinkage-thresholding algorithm for linear inverse
  problems.
\newblock \emph{SIAM Journal on Imaging Sciences} 2, 183--202.
\newblock \url{https://doi.org/10.1137/080716542}
\bibAnnoteFile{beck2009fast}

\bibitem[{Benson et~al.(2016)Benson, Gleich, and Leskovec}]{benson2016higher}
Benson, A.~R., Gleich, D.~F., and Leskovec, J. (2016).
\newblock Higher-order organization of complex networks.
\newblock \emph{Science} 353, 163--166.
\newblock \url{https://doi.org/10.1126/science.aad9029}
\bibAnnoteFile{benson2016higher}

\bibitem[{Boyd et~al.(2004)Boyd, Boyd, and Vandenberghe}]{boyd2004convex}
Boyd, S., Boyd, S.~P., and Vandenberghe, L. (2004).
\newblock \emph{Convex optimization} (Cambridge university press)
\bibAnnoteFile{boyd2004convex}

\bibitem[{B{\"u}hler and Hein(2009)}]{buhler2009spectral}
B{\"u}hler, T. and Hein, M. (2009).
\newblock Spectral clustering based on the graph p-laplacian.
\newblock In \emph{International Conference on Machine Learning}. 81--88.
\newblock \url{https://doi.org/10.1145/1553374.1553385}
\bibAnnoteFile{buhler2009spectral}

\bibitem[{Chambolle and Pock(2011)}]{chambolle2011first}
Chambolle, A. and Pock, T. (2011).
\newblock A first-order primal-dual algorithm for convex problems with
  applications to imaging.
\newblock \emph{Journal of mathematical imaging and vision} 40, 120--145.
\newblock \url{https://doi.org/10.1007/s10851-010-0251-1}
\bibAnnoteFile{chambolle2011first}

\bibitem[{Chambolle and Pock(2016{\natexlab{a}})}]{chambolle2016introduction}
Chambolle, A. and Pock, T. (2016{\natexlab{a}}).
\newblock An introduction to continuous optimization for imaging.
\newblock \emph{Acta Numerica} 25, 161--319.
\newblock \url{https://doi.org/10.1017/S096249291600009X}
\bibAnnoteFile{chambolle2016introduction}

\bibitem[{Chambolle and Pock(2016{\natexlab{b}})}]{chambolle2016ergodic}
Chambolle, A. and Pock, T. (2016{\natexlab{b}}).
\newblock On the ergodic convergence rates of a first-order primal--dual
  algorithm.
\newblock \emph{Mathematical Programming} 159, 253--287.
\newblock \url{https://doi.org/10.1007/s10107-015-0957-3}
\bibAnnoteFile{chambolle2016ergodic}

\bibitem[{Chang(2016)}]{chang2016spectrum}
Chang, K.-C. (2016).
\newblock Spectrum of the 1-laplacian and cheeger's constant on graphs.
\newblock \emph{Journal of Graph Theory} 81, 167--207.
\newblock \url{https://doi.org/10.1002/jgt.21871}
\bibAnnoteFile{chang2016spectrum}

\bibitem[{Chang et~al.(2017)Chang, Shao, and Zhang}]{chang2017nodal}
Chang, K.-C., Shao, S., and Zhang, D. (2017).
\newblock Nodal domains of eigenvectors for 1-laplacian on graphs.
\newblock \emph{Advances in Mathematics} 308, 529--574.
\newblock \url{https://doi.org/10.1016/j.aim.2016.12.020}
\bibAnnoteFile{chang2017nodal}

\bibitem[{Chitra and Raphael(2019)}]{chitra2019random}
Chitra, U. and Raphael, B. (2019).
\newblock Random walks on hypergraphs with edge-dependent vertex weights.
\newblock In \emph{International Conference on Machine Learning}. 1172--1181.
\newblock \url{http://proceedings.mlr.press/v97/chitra19a.html}
\bibAnnoteFile{chitra2019random}

\bibitem[{Ene and Nguyen(2015)}]{ene2015random}
Ene, A. and Nguyen, H. (2015).
\newblock Random coordinate descent methods for minimizing decomposable
  submodular functions.
\newblock In \emph{International Conference on Machine Learning}. 787--795.
\newblock \url{http://proceedings.mlr.press/v37/ene15.html}
\bibAnnoteFile{ene2015random}

\bibitem[{Fu et~al.(2022)Fu, Zhao, and Bian}]{fu2022p}
Fu, G., Zhao, P., and Bian, Y. (2022).
\newblock p-laplacian based graph neural networks.
\newblock In \emph{International Conference on Machine Learning}. 6878--6917.
\newblock \url{https://proceedings.mlr.press/v162/fu22e.html}
\bibAnnoteFile{fu2022p}

\bibitem[{Hayashi et~al.(2020)Hayashi, Aksoy, Park, and
  Park}]{hayashi2020hypergraph}
Hayashi, K., Aksoy, S.~G., Park, C.~H., and Park, H. (2020).
\newblock Hypergraph random walks, laplacians, and clustering.
\newblock In \emph{Conference on Information and Knowledge Management}.
  495--504.
\newblock \url{https://doi.org/10.1145/3340531.3412034}
\bibAnnoteFile{hayashi2020hypergraph}

\bibitem[{Hein and B{\"u}hler(2010)}]{hein2010inverse}
Hein, M. and B{\"u}hler, T. (2010).
\newblock An inverse power method for nonlinear eigenproblems with applications
  in 1-spectral clustering and sparse {PCA}.
\newblock \emph{Advances in Neural Information Processing Systems} 23.
\newblock \url{https://dl.acm.org/doi/10.5555/2997189.2997284}
\bibAnnoteFile{hein2010inverse}

\bibitem[{Hein and Setzer(2011)}]{hein2011beyond}
Hein, M. and Setzer, S. (2011).
\newblock Beyond spectral clustering - tight relaxations of balanced graph
  cuts.
\newblock \emph{Advances in neural information processing systems} 24.
\newblock \url{https://dl.acm.org/doi/10.5555/2986459.2986723}
\bibAnnoteFile{hein2011beyond}

\bibitem[{Hein et~al.(2013)Hein, Setzer, Jost, and Rangapuram}]{hein2013total}
Hein, M., Setzer, S., Jost, L., and Rangapuram, S.~S. (2013).
\newblock The total variation on hypergraphs-learning on hypergraphs revisited.
\newblock \emph{Advances in Neural Information Processing Systems} 26.
\newblock \url{https://dl.acm.org/doi/10.5555/2999792.2999883}
\bibAnnoteFile{hein2013total}

\bibitem[{Jegelka et~al.(2013)Jegelka, Bach, and Sra}]{jegelka2013reflection}
Jegelka, S., Bach, F., and Sra, S. (2013).
\newblock Reflection methods for user-friendly submodular optimization.
\newblock \emph{Advances in Neural Information Processing Systems} 26.
\newblock \url{https://dl.acm.org/doi/10.5555/2999611.2999758}
\bibAnnoteFile{jegelka2013reflection}

\bibitem[{Leskovec et~al.(2020)Leskovec, Rajaraman, and
  Ullman}]{leskovec2020mining}
Leskovec, J., Rajaraman, A., and Ullman, J.~D. (2020).
\newblock \emph{Mining of massive data sets}.
\newblock \url{https://doi.org/10.1017/CBO9781139924801}
\bibAnnoteFile{leskovec2020mining}

\bibitem[{Lewis et~al.(2004)Lewis, Yang, Russell-Rose, and Li}]{lewis2004rcv1}
Lewis, D.~D., Yang, Y., Russell-Rose, T., and Li, F. (2004).
\newblock {RCV1: A} new benchmark collection for text categorization research.
\newblock \emph{Journal of Machine Learning Research} 5, 361--397.
\newblock \url{https://dl.acm.org/doi/10.5555/1005332.1005345}
\bibAnnoteFile{lewis2004rcv1}

\bibitem[{Li et~al.(2018)Li, He, and Zhu}]{li2018tail}
Li, J., He, J., and Zhu, Y. (2018).
\newblock E-tail product return prediction via hypergraph-based local graph
  cut.
\newblock In \emph{International Conference on Knowledge Discovery \& Data
  Mining}. 519--527.
\newblock \url{https://doi.org/10.1145/3219819.3219829}
\bibAnnoteFile{li2018tail}

\bibitem[{Li and Milenkovic(2017)}]{li2017inhomogeneous}
Li, P. and Milenkovic, O. (2017).
\newblock Inhomogeneous hypergraph clustering with applications.
\newblock \emph{Advances in Neural Information Processing Systems} 30.
\newblock \url{https://dl.acm.org/doi/10.5555/3294771.3294991}
\bibAnnoteFile{li2017inhomogeneous}

\bibitem[{Li and Milenkovic(2018)}]{li2018submodular}
Li, P. and Milenkovic, O. (2018).
\newblock Submodular hypergraphs: p-laplacians, cheeger inequalities and
  spectral clustering.
\newblock In \emph{International Conference on Machine Learning}. 3014--3023.
\newblock \url{http://proceedings.mlr.press/v80/li18e.html}
\bibAnnoteFile{li2018submodular}

\bibitem[{Liu et~al.(2021)Liu, Veldt, Song, Li, and Gleich}]{liu2021strongly}
Liu, M., Veldt, N., Song, H., Li, P., and Gleich, D.~F. (2021).
\newblock Strongly local hypergraph diffusions for clustering and
  semi-supervised learning.
\newblock In \emph{International World Wide Web Conference}. 2092--2103.
\newblock \url{https://doi.org/10.1145/3442381.3449887}
\bibAnnoteFile{liu2021strongly}

\bibitem[{Lov{\'a}sz(1983)}]{lovasz1983submodular}
Lov{\'a}sz, L. (1983).
\newblock Submodular functions and convexity.
\newblock In \emph{Mathematical programming The state of the art}. 235--257.
\newblock \url{https://doi.org/10.1007/978-3-642-68874-4_10}
\bibAnnoteFile{lovasz1983submodular}

\bibitem[{Nesterov(1983)}]{nesterov1983method}
Nesterov, Y.~E. (1983).
\newblock A method for solving the convex programming problem with convergence
  rate o (1/k\^{} 2).
\newblock In \emph{Dokl. akad. nauk Sssr}. vol. 269, 543--547
\bibAnnoteFile{nesterov1983method}

\bibitem[{Rockafellar(1970)}]{rockafellar1970convex}
Rockafellar, R.~T. (1970).
\newblock \emph{Convex analysis}, vol.~18.
\newblock \url{https://doi.org/10.1515/9781400873173}
\bibAnnoteFile{rockafellar1970convex}

\bibitem[{Schaub et~al.(2021)Schaub, Zhu, Seby, Roddenberry, and
  Segarra}]{schaub2021signal}
Schaub, M.~T., Zhu, Y., Seby, J.-B., Roddenberry, T.~M., and Segarra, S.
  (2021).
\newblock Signal processing on higher-order networks: Livin’on the edge...
  and beyond.
\newblock \emph{Signal Processing} 187, 108149.
\newblock \url{https://doi.org/10.1016/j.sigpro.2021.108149}
\bibAnnoteFile{schaub2021signal}

\bibitem[{Szlam and Bresson(2010)}]{szlam2010total}
Szlam, A. and Bresson, X. (2010).
\newblock Total variation and cheeger cuts.
\newblock In \emph{International Conference on Machine Learning}. 1039--1046.
\newblock \url{https://dl.acm.org/doi/10.5555/3104322.3104454}
\bibAnnoteFile{szlam2010total}

\bibitem[{Tudisco and Hein(2018)}]{tudisco2018nodal}
Tudisco, F. and Hein, M. (2018).
\newblock A nodal domain theorem and a higher-order cheeger inequality for the
  graph $ p $-laplacian.
\newblock \emph{Journal of Spectral Theory} 8, 883--908.
\newblock \url{https://doi.org/10.4171/JST/216}
\bibAnnoteFile{tudisco2018nodal}

\bibitem[{Tudisco et~al.(2018)Tudisco, Mercado, and
  Hein}]{tudisco2018community}
Tudisco, F., Mercado, P., and Hein, M. (2018).
\newblock Community detection in networks via nonlinear modularity
  eigenvectors.
\newblock \emph{SIAM Journal on Applied Mathematics} 78, 2393--2419.
\newblock \url{https://doi.org/10.1137/17M1144143}
\bibAnnoteFile{tudisco2018community}

\bibitem[{Veldt et~al.(2020)Veldt, Benson, and Kleinberg}]{veldt2020hypergraph}
Veldt, N., Benson, A.~R., and Kleinberg, J. (2020).
\newblock Hypergraph cuts with general splitting functions.
\newblock \emph{arXiv preprint arXiv:2001.02817}
  \url{https://arxiv.org/abs/2001.02817}
\bibAnnoteFile{veldt2020hypergraph}

\bibitem[{Von~Luxburg(2007)}]{von2007tutorial}
Von~Luxburg, U. (2007).
\newblock A tutorial on spectral clustering.
\newblock \emph{Statistics and Computing} 17, 395--416.
\newblock \url{https://doi.org/10.1007/s11222-007-9033-z}
\bibAnnoteFile{von2007tutorial}

\bibitem[{Wagner and Wagner(1993)}]{wagner1993between}
Wagner, D. and Wagner, F. (1993).
\newblock Between min cut and graph bisection.
\newblock In \emph{International Symposium on Mathematical Foundations of
  Computer Science}. 744--750.
\newblock \url{https://doi.org/10.1007/3-540-57182-5_65}
\bibAnnoteFile{wagner1993between}

\bibitem[{Yoshida(2019)}]{yoshida2019cheeger}
Yoshida, Y. (2019).
\newblock Cheeger inequalities for submodular transformations.
\newblock In \emph{ACM-SIAM Symposium on Discrete Algorithms}. 2582--2601.
\newblock \url{https://doi.org/10.1137/1.9781611975482.160}
\bibAnnoteFile{yoshida2019cheeger}

\bibitem[{Zhu et~al.(2021)Zhu, Li, and Segarra}]{zhu2021co}
Zhu, Y., Li, B., and Segarra, S. (2021).
\newblock Co-clustering vertices and hyperedges via spectral hypergraph
  partitioning.
\newblock In \emph{European Signal Processing Conference}. 1416--1420.
\newblock \url{https://doi.org/10.23919/EUSIPCO54536.2021.9616223}
\bibAnnoteFile{zhu2021co}

\bibitem[{Zhu et~al.(2022)Zhu, Li, and Segarra}]{zhu2022hypergraphs}
Zhu, Y., Li, B., and Segarra, S. (2022).
\newblock Hypergraphs with edge-dependent vertex weights: Spectral clustering
  based on the 1-laplacian.
\newblock In \emph{International Conference on Acoustics, Speech and Signal
  Processing}. 8837--8841.
\newblock \url{https://doi.org/10.1109/ICASSP43922.2022.9746363}
\bibAnnoteFile{zhu2022hypergraphs}

\bibitem[{Zhu and Segarra(2022)}]{zhu2022ans}
Zhu, Y. and Segarra, S. (2022).
\newblock Hypergraph cuts with edge-dependent vertex weights.
\newblock \emph{Applied Network Science} 7, 45.
\newblock \url{https://doi.org/10.1007/s41109-022-00483-x}
\bibAnnoteFile{zhu2022ans}

\end{thebibliography}




\end{document}